\RequirePackage{fix-cm}
\documentclass[smallcondensed]{svjour3}     % onecolumn (ditto)
\smartqed  % flush right qed marks, e.g. at end of proof
\usepackage{graphicx}
\usepackage{epstopdf}
\usepackage{mathptmx}      
\usepackage{latexsym}
\usepackage{amsmath}
\usepackage{amssymb}
\usepackage{amsfonts}
\usepackage{url}
\usepackage{calc}
\usepackage{enumerate}
\usepackage{color}
\journalname{Preprint}

\begin{document}

\title{Detecting truth,  just on parts \thanks{Partially supported by the Spanish Research Project MTM2017-88796-P.}}


\author{  Zolt\'an Kov\'acs \and Tom\'as Recio \and M.~Pilar V\'elez }

\authorrunning {Kov\'acs, Recio and  V\'elez }% if too long for running head

\institute{
           Zolt\'an Kov\'acs (ORCID: 0000-0003-2512-5793) \at
              Private P\"adagogische Hochschule der Di\"ozese Linz, Salesianumweg 3, 4020 Linz, Austria\\
                            \email{zoltan@geogebra.org} % \\
               \and
           Tom\'as Recio (ORCID:  0000-0002-1011-295X) \at
              Dpto.~Matem\'aticas, Estad\'\i stica y Computaci\'on, Universidad de Cantabria, Av. de los Castros, s/n, 39071 Santander, Spain \\
               \email{tomas.recio@unican.es}  
               \and
        M. Pilar V\'elez (ORCID:0000-0002-5724-4300) \at
              Dpto.~Ingeniería Industrial, Universidad Antonio de Nebrija, c/ Pirineos, 55, 28040 Madrid, Spain\\
               \email{pvelez@nebrija.es}  %  \\ }
}

\date{Received: date / Accepted: date}

\maketitle

\begin{abstract}
We introduce and discuss,  through a computational algebraic geometry approach, the automatic reasoning handling of propositions that are simultaneously true and false over some relevant collections of instances.  A rigorous, algorithmic criterion is presented for detecting such cases,  and its performance is exemplified through the implementation of this test on the dynamic geometry program \textit{GeoGebra}.

\keywords{automatic deduction in geometry,  automatic geometry theorem proving \and automatic geometry theorem discovery \and elementary geometry \and Gr\"obner basis \and zero divisor \and true on parts, false on parts \and true on components \and dynamic geometry}
\end{abstract}

\section{Introduction}
\label{intro}
In this paper, we deal with a particular issue in automated proving and discovery of theorems in elementary geometry by algebraic geometry methods, namely, the case of statements that are neither true nor false (in some specific sense we will describe in detail below). Very roughly, the algebraic geometry approach to automated reasoning in geometry proceeds by translating a geometric statement $\{H\Rightarrow T\}$ into polynomial expressions, after adopting a coordinate system. Then, the geometric instances verifying the hypotheses (respectively the thesis) can be represented as the solution of a system of polynomial equations $V(H)=\{h_1=0,\dots,h_r=0\}$ (respectively $V(T)=\{f=0\}$),  describing the hypotheses (resp. the thesis) variety. 

Thus, when $V(H) \subseteq V(T)$ we can say that the theorem is \emph{always true}, i.e.~true for all instances of the hypotheses. But this fact rarely happens, even for well established theorems, because the algebraic translation of the geometric construction described by the hypotheses usually forgets explicitly excluding some  degenerate cases (say, when a triangle collapses to a line,  when  two points defining a line become coincident, etc.) where the given statement fails.  Unfortunately, many of these cases are not intuitively obvious and they are hard to detect \emph{a priori}; even if detected, it happens that introducing negative hypotheses (i.e.~ declaring that some geometric elements in the given statement should not verify a certain relation among them, such as the non-collinearity of the three vertices of a triangle) might involve some subtle issues, cf. \cite{DR}, \cite{LPR}.

Thus, a delicate, but more useful,  approach for automated reasoning consists in exhibiting, first, a collection of independent variables modulo $H$, so that no polynomial relation among them holds over the whole $V(H)$. That is, identifying a collection of parameters describing the coordinates of the free elements in the given geometric statement. As we will show in the next Section, such identification involves quite delicate issues;  but let us concentrate here in providing just a rough description of the global procedure. Now, once such independent variables have been selected, the irreducible components of $V(H)$ where these variables do not remain independent are assumed to describe \emph{degenerate} instances and are, in some sense, negligible. Accordingly,  a statement is called \emph{generally true} if the thesis holds, at least, over all the non-degenerate components. 

On the other hand, if over each non-degenerate component the thesis does not identically vanish, the statement is labeled as \emph{generally false}: this includes the \emph{always false} case, where the thesis does not hold at all,  i.e.~when $\{f\neq 0\}$  over every point of  $V(H)$, and also the case when the equality $\{f = 0\}$ holds true just at some negligible set, i.e.~over  a degenerate component of $V(H)$) or over a proper subvariety of a non-degenerate component.  See  Fig.~\ref{fig1} for a graphical representation of all these terms and the relations among them. A more detailed description of this quite established terminology (with small variants) can be consulted, for instance, at \cite{RV99}, \cite{CLS} or \cite{ZWS}.  

Let us point out that it is within this more flexible concept of truth that the algebraic approach to automated theorem reasoning has shown all its capacity to verify and to discover thousands of geometric statements, either elementary or sophisticated. And, behind this success story lies the existence of algebraic methods to test the general truth or failure of a statement without actually having to explicitly compute the decomposition of the given hypotheses variety on components, without finding which of these are degenerate or not, without verifying, one by one,  over which components the thesis holds\dots. Thus, avoiding the use of costly primary decomposition algorithms is an implicit, but strong,  restriction that the reader must keep in mind to truly understand what follows.

\begin{figure}[h]
\includegraphics[scale=0.9]{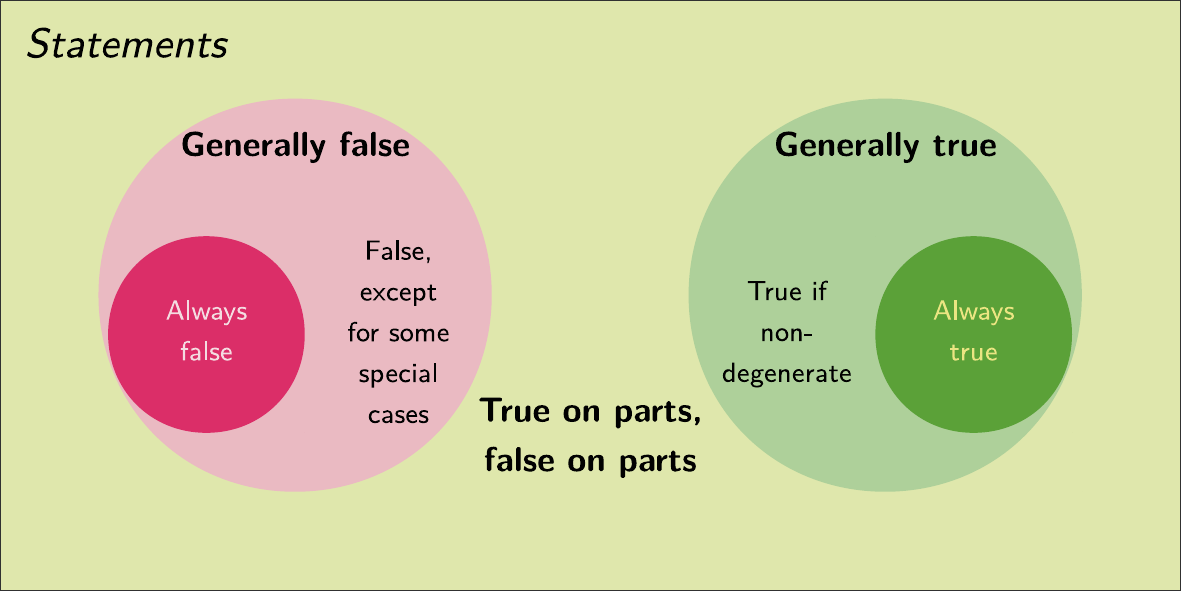}\caption{Different types of truth for a geometry statement}
\label{fig1}       
\end{figure}

It follows from the previous definition that  to be generally true and to be generally false are incompatible; there can be no statement having both properties at once.  On the other hand---and this is the object of interest in this paper---there are statements which happen to be,  simultaneously, not generally true and not generally false. That is,  statements that are both false over some non-degenerate component and that are true over some other non-degenerate component, i.e. statements that are \emph{true, just on some components}. 
We have decided---for the better comprehension of this notion by general users of dynamic geometry programs implementing this feature, such as {\it GeoGebra}---to label such statements in a more colloquial way, as  statements \emph{true on parts, false on parts}.
The specific interest of this concept, in the context of automated reasoning in geometry, is   briefly justified in Section \ref{sec1}, where the basic definitions  and other fundamental issues are precisely introduced.

Some explicit examples of this situation have been provided by one of the authors of this paper as early as 1998 in \cite{R}; further examples appeared in \cite{tapas} or  in sections 2.3 and 2.4. of \cite{RV99}, reproducing in English the example, in Spanish, from \cite{R}). A recent paper, specifically devoted to discuss and to present such cases is \cite{BR}.  But it was not, until even more recently, at \cite{ZWS}, that a new terminology to describe such cases has been introduced, labelling as \emph{generally true on components} or, simply,  as \emph{true on components}, those statements that are true in some,  and false in some other,  non-degenerate components: i.e.~statements that, according to our terminology,  are simultaneously \emph{true on parts, false on parts}. Moreover, \cite{ZWS} presents an algorithmic test to check this property.
 
Now, since the idea of \emph{true on components}, or \emph{true on parts, false on parts}, is based on the concepts of degeneracy and of irreducible component, it follows that both the choice of the field over which the prime decomposition is performed and the choice of the independent variables --which determine which components are to be considered as degenerate-- could be essential.  

Attempting to address these options, Section \ref{sec1} in this paper recalls some fundamental notions and argues why we have decide adopting, as our algebraic geometry framework, the consideration of statements defined over some base field $K$, but thinking of the associated algebraic varieties over an algebraically closed extension $K \subseteq L$. This framework, although quite classical (c.f.~\cite{Z-S}) and already quite common in the automated theorem proving context (c.f.~\cite{Chou}, \cite{CLS},  \cite{RV99}), is  more general than the one used in the paper \cite{ZWS}.  In fact, \cite{ZWS} considers just one algebraically closed field, both as base field and as field of solutions. 
  
 Moreover, Section \ref{sec1} elaborates a restricted---yet quite relevant---notion of non-degeneracy, by adopting always a set of independent variables of cardinal equal to the dimension of the hypotheses variety. This condition is a kind of intuitive translation of the expected fact that, for given values of the independent coordinates, the hypotheses variety contains just a finite number of configurations (as in \cite{CLS}, Chapter 6, Section 4, Definition 4; see also the Dimensionality Restriction in \cite{Chou};  or the need, expressed in \cite{RV99} or \cite{DR}, to include the equality of the dimension and the number of free variables to obtain sound results). Again, by adopting this convention we differ from the approach in \cite{ZWS}.
 
In this specific algebro-geometric framework, Theorem \ref{testRV} of Section  \ref{sec2} provides a new, simpler way,  of testing if a statement is true and false on parts, by just detecting if a pair of elimination ideals are zero or not. As a consequence, it is shown the somehow surprising result (c.f.~Corollary \ref{extensionfield}) that the notion of \emph{true on parts, false on parts} does not depend on the base field being considered.  Moreover, Section \ref{sec2} includes an extension to our more general framework of the main result of \cite{ZWS}, by providing  a direct proof  (cf. Theorem \ref{equivalent}) of the equivalence (for degenerate components of the special kind) of our test (c.f.~Theorem 1) and that of \cite{ZWS}, here labeled as Theorem \ref{testZWS}.  Section 3 also provides a counterexample to this equivalence (Example \ref{ex1}) if the mentioned  degeneracy concept is not fulfilled;  and shows how this example helps to explain some discrepancy,  mentioned at the paper \cite{ZWS},  with a previous result of \cite{MR}.
 
Finally, Section \ref{sec3} shows some examples on how to actually deal with the concept of \emph{true on parts, false on parts},  by performing our test  as  implemented in the dynamic geometry software \textit{GeoGebra}. The paper concludes with a reflection on the relevance of such automated reasoning tools capable of handling this subtle idea of truth (on parts!).

%%%%%%%%%%%%%%%%
\section{Motivation for a new framework}
\label{sec1}

It might seem, at first glance, that the \emph{true on parts, false on parts} case is a useless oddity of the automated proving method, arising just in some very artificial statements without a real geometric motivation. But it is not so. 
 
  In fact, as described rigorously in \cite{DR},  Section 3, Proposition 2 (although without explicitly providing a concrete naming for this situation), arriving to a \emph{true only on some components} case means  ``yielding a warning sign for the need to factorize".  Luckily,  such warning sign---even without actually involving algorithmic factorization, but human reflection, instead---leads to the discovery of new (surely for the user, but, sometimes,  for the scientific community as well) theorems. For instance, it has already allowed one of the authors of this paper to work out some contributions, such as a converse to Varignon Theorem \cite{BR}, or the generalization of the Steiner-Lehmus Theorem (c.f.~Example 9 in \cite{DR},  fully described in \cite{LRV}). 
 In summary: we think it could be quite rewarding to devote some efforts to understand better the \emph{true on parts, false on parts} concept, and this is the goal of the current section.
 
 Let us first start analyzing a simple example, a modified version of the Example in Section 3 from \cite{BR}. Consider  points $A (0,0), B (2,0)$ in the plane and construct circles $c= (x-0)^2+(y-0)^2-3$ and $d=(x-2)^2+(y-0)^2-3$, i.e.~circle $c$ is centered at $A$ and circle $d$ is centered at $B$ and both have the same radius $r$, where $r=\sqrt 3$.  Finally, we consider the two points of intersection of these circles, namely, $E (u, v)$ and $F (m, n)$, so the hypotheses are 
 $$\begin{array}{l}
  (u-0)^2+(v-0)^2-3,(u-2)^2+(v-0)^2-3,\\
  (m-0)^2+(n-0)^2-3, (m-2)^2+(n-0)^2-3. \end{array} $$
 
 \noindent Then,  the thesis states the parallelism of the lines $AE$ and $BF$, that is, the vanishing of the polynomial
 $$(u-0)\cdot(n-0)-(v-0)\cdot(m-2).$$
 \noindent The ideal of hypotheses is clearly zero-dimensional, so there are no independent variables, nor degenerate components. Its primary components, over the rationals, are\\
 $$\begin{array}{l}
 \left<v-n, (m-2)^2+n^2-3, (u-2)^2+v^2-3, m^2+n^2-3, u^2+v^2-3\right>, \\
  \left<v+n, (m-2)^2+n^2-3, (u-2)^2+v^2-3, m^2+n^2-3, u^2+v^2-3\right>\end{array}$$
\noindent and it easy to check that the thesis is false over the first one and true over the second. This a clear, simple example of a \emph{true on components} statement, arising in an elementary geometry context. 
 
\begin{figure}[h]
\includegraphics[width=0.8\textwidth]{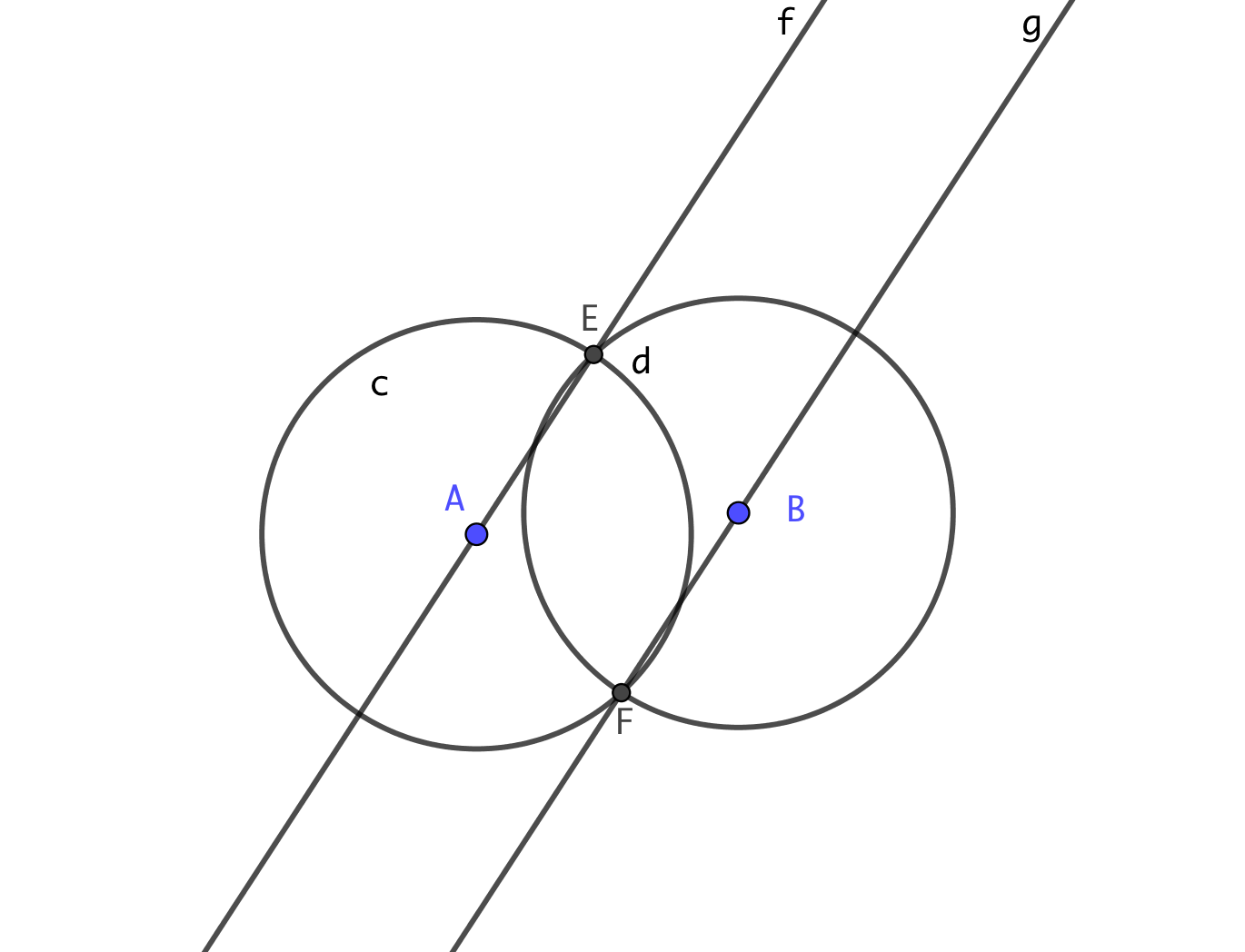}\caption{A simple example}
\label{figsimple}       
\end{figure}

Even considering this is a very basic example of a \emph{true on parts, false on parts} statement, it already allows us to emphasize different issues that motivate our choices in this paper.

First of all, it highlights the relevant role of the base field (the field of coefficients where the hypotheses and thesis equations are described, i.e.~$\mathbb{Q}$, in this example).  It is obvious that if we would have chosen instead,  as base field, $\mathbb{Q}(\sqrt 2)$, the hypotheses ideal could have been the following precise description of the two constructed points $E$ and $F$, $\left<u-1,v-\sqrt 2,m-1,n+\sqrt 2\right>$, clearly verifying the thesis $u\cdot n-v\cdot(m-2)$.  The statement would have been a true one, then.

Moreover, even if keeping the original hypotheses ideal, it is clear that its primary decomposition depends on the coefficients field for the ring where the components are computed. A trivial example is the ideal $\left<x^2-2\right>$, that is irreducible over $\mathbb{Q}[x]$, but has two components $\left<x-\sqrt 2\right>$ and $\left<x+\sqrt 2\right> $ over $\mathbb{Q}(\sqrt 2)[x]$. Likewise, the hypotheses ideal
\begin{equation*}
\begin{array}{l}
\left<(u-0)^2+(v-0)^2-3,(u-2)^2+(v-0)^2-3,\right.\\
\left.(m-0)^2+(n-0)^2-3, (m-2)^2+(n-0)^2-3\right>
\end{array}
\end{equation*}
has two components over $\mathbb{Q}[u,v,m,n]$, but four over $\mathbb{Q}(\sqrt 2)[u,v,m,n]$, namely, 
$$\left<n -\alpha, v -\alpha, (m-2)^2+n^2-3, (u-2)^2+v^2-3, m^2+n^2-3, u^2+v^2-3)\right>,$$
\noindent  considering all possible sign values for $\alpha=\sqrt 2$. Notice that, as over $\mathbb{Q}$, the statement would also remain \emph{true on parts, false on parts}over $\mathbb{Q}(\sqrt 2)$, because here the thesis would be false over two components and true over the other two.

Obviously, the idea of accepting, as input, polynomials over field extensions of the rationals, could avoid some---but not all--- \emph{true on parts, false on parts} cases, by allowing the user to be more specific regarding the hypotheses data, but it also  implies the symbolic manipulation of quite complicated expressions and it is, in practice, not realistic. 

These considerations are behind our generalized approach to the \emph{true on parts, false on parts} concept, as developed in the next section, and to the extension of the Zhou-Wang-Sun  test (\cite{ZWS}) to this framework, see Theorem  \ref{equivalent}. 

On the other hand, the above example does not requires any discussion about the idea of degenerate components, since this concept is linked to the idea of independent parameters for our hypotheses set, and there are none in this example, as the dimension of the hypotheses variety is zero. So here all components can be thought as non-degenerate. But, in general, it is well known, since long ago, that the precise choice of a meaningful set of independent parameters and,  correspondingly, the notion of degenerate components, is an involved issue.  And, since the definition of \emph{true on parts, false on parts} involves the truth and failure of a statement over some non-degenerate components, both the idea of independence of parameters and of degeneracy are concepts that can not be avoided in this context.  

Let us recall (c.f.~\cite{CLS}, particularly Chapter 9, Section 5, for precise definitions and basic results) that a collection of variables is considered to be independent over the hypotheses ideal if there is no polynomial in these variables alone belonging to the ideal; and that a component is labeled as degenerate if over the component the chosen independent variables verify some non-trivial relation. This notion intends to reflect the idea of \emph{free} variables for our geometric statement; moreover, it is related to the concept of Hilbert Dimension (i.e.~the dimension) of the ideal, since this dimension coincides with the largest cardinal of a set of free variables modulo the given ideal (but notice that not every set of free variables can be enlarged to one with maximum cardinality).  

There are, in general, many possible sets of independent variables for the ideal $H$, even concerning maximal sets of independent variables or  even considering  maximum-size  maximal sets. For example, if $H=\left<x\cdot y\right>$, both $\{x\}$ and $\{y\}$ are maximal and maximum-size sets of independent variables; and, if $H=\left<x\cdot y, ~ x\cdot z\right>$, then both $\{x\}$ and $\{y, z\}$ are maximal sets of independent variables, but only the second has cardinality equal to the dimension of $H$. 

When dealing with geometric statements it seems logical to take as independent variables the coordinates of free points in the geometric construction we are dealing with. In most cases this ``intuitively'' maximal set of independent variables is maximum-size, but there are examples in which the coordinates of free points in the geometric construction do not provide a maximum-size set of independent variables. See,  for instance,  Example 2 in \cite[p.~72]{RV99}: the number of coordinates of free points in the chosen geometric construction is 5, but the Krull dimension of $H$ is 6. 

 Another typical example of the difficulties involved in handling this concept is Example 7 in \cite{DR}, concerning Euler's formula regarding the radii of the inner and outer circles of a triangle with vertices $(-1,0),(1,0), (u[1],u[2])$. Here the dimension of the hypotheses variety is expected to be 2 (referring to the two coordinates of the only free vertex of the triangle), but applying the algebraic definition it turns out to be three\dots , unless it is explicitly required, as new hypothesis, that $(u[1],u[2])$ does not lie in the $x$-axis! This quite common problem---related, as mentioned above, to the difficult  \emph{a priori} control and detail of all geometric degeneracies---is already considered in the basic reference of \cite{Chou}.  
 
 Despite of these difficulties, we think---as argued and documented in the introduction---that the closest choice reflecting, in most cases, human intuition, is that of considering a maximum-size set of independent variables as the one best related to the notion of \emph{true on parts, false on parts}.  The precise definition will be developed in the next Section.  But we are aware that  the idea of \emph{true on parts, false on parts} depends on the adopted formulation of non-degeneracy in each particular case and, therefore, on the selected set of free variables for the given statement.
 In the following Section  \ref{sec2} we will  show some precise statements and counterexamples concerning the consequences of this choice.   
 
 As a toy example, consider the following (artificial) statement:  take as hypothesis set the union of the two axes in the plane, i.e.~the set of points verifying $xy=0$. Its dimension is one, so we might consider that the geometry of the problem involved in this formulation requires just one free variable, say, $x$ and,  thus, we could label  the $x$-axis as the only non-degenerate component in this problem. So, if the thesis is $y=0$, we could conclude that it is generally true, since it holds over the $x$-axis.  But if we consider, instead, as the only non-degenerate component,  the $y$-axis, then $y=0$ would be generally false. And if we choose to consider both $x, y$ as two non-degenerate parameters simultaneously ruling our construction, then the thesis $y=0$ will be true on components, since it will hold on the $x$-axis and fail on the $y$-axis.  
 
% Another, more elaborated example, could be the following: take $H=\{xyz,xyw\}$, so $V(H)= {x=0} \cup {y=0} \cup {z=0,w=0}$, and the thesis $T:=y=0$. Then both $\{x,z,w\}$ and $\{y,z,w\}$ are sets of independent variables  for $H$, both of maximum cardinality.  For the first set of variables, only the component ${y=0}$ is non-degenerate, so the thesis is generally true. For the second set of variables, only the component ${x=0}$ will be non-degenerate, so the thesis will be generally false.

\section{Statements true on parts, false on parts}
\label{sec2}
As above, let us consider  an algebraically translated statement  $\{H  \Rightarrow  T\}$, where $H$ stands for the equations describing the geometric construction of the hypotheses and $T$ describes the thesis. By abuse of notation, we will denote also by   $H$ and $ T$    the ideals generated by the polynomials involved in the equations describing the statement. 
 In what follows we will suppose that $H=\left<h_1,
\dots , h_r\right>$ and $T=\left<f\right>$ are the hypotheses and thesis ideals  in a polynomial ring $K[X]$, $X=\{x_{1}, \dots, x_{n}\}$, where the variables $X=\{x_{1}, \dots, x_{n}\}$ refer to the coordinates involved in the algebraic description of the hypotheses, over a base field $K$. 

We will deal with another field $L$, an algebraically closed extension on $K$ (for instance $L={\mathbb C}$ and $K={\mathbb Q}$), and the geometric instances verifying the hypotheses (respectively, the thesis) of the statement will be considered as the algebraic variety $V(H)$ (respectively, $V(T)$) in the affine space  $L^n$. Therefore,  $V(H)$ and $V(T)$ are algebraic varieties of $L^n$, but defined over $K$.   Thus, we are working here with two different fields: the one of coefficients for the algebraic description of the geometric setting and the one where the solutions of the equations are to be considered at. In fact, most elementary geometry constructions can be translated into algebraic equations with rational coefficients, while the algorithms we will consider (for its higher performance) to work with these equations are those of algebraic geometry over an algebraically closed field (i.e not over the rationals or over the reals).

In general we will suppose that  $Y=\{x_{1}, \dots, x_{d}\}$ ($0\leq d\leq n$) is a \emph{maximum-size set of independent variables} for the hypotheses ideal $H$, that is
\begin{enumerate}[(i)]
\item $H\cap K[Y]=\left<0\right>$, and
\item for any other set of variables $Z\subset X$ with $r>d$ elements, $H\cap K[Z]\neq\left<0\right>$. 
\end{enumerate}
Consequently, the Krull dimension of the hypothesis variety $V(H)$ must be $d$.

Following the notation above we recall the following definitions which are usual in the literature about the algebraic geometry approach to automated reasoning  in geometry \cite{RV99}. 

\begin{definition}
Let  $\{H  \Rightarrow  T\}$ be a geometric statement and fix a set $Y=\{x_{1}, \dots, x_{d}\}$ of independent variables for the hypotheses ideal $H$.
\begin{itemize}
\item The statement is \emph{generally true} if the thesis $f$ vanishes on all non-degenerate $K$-components of the hypotheses variety $V(H)$.
\item The  statement is \emph{generally false} if the thesis $f$ vanishes on none of the non-degenerate $K$-components of the hypotheses variety $V(H)$.
\end{itemize}
\end{definition}

The related concept of ``generally true on components'' statements was introduced by Zhou, Wang and Sun  in \cite{ZWS} but in a slightly different, less general,  context, assuming $K=L$,  algebraically closed. Here we mimic this idea over our more general framework, as follows: 

\begin{definition}
Let  $\{H  \Rightarrow  T\}$ be a geometric statement formulated over $K$. The statement is labelled as \emph{true on parts, false on parts} if the thesis $f$ vanishes on some but not all non-degenerate $K$-components of the hypotheses variety $V(H)$ in $L^n$. That is, if it is neither generally true nor generally false.
\end{definition}

In \cite{RV99} the reader can find algorithmic criteria for the generally true and the generally false cases,  by means of some elimination ideals with respect to independent variables of the free points coordinates. Moreover, in this reference it is shown how to derive, from these elimination ideals,  some conditions to discover new theorems. This approach has been implemented in the widely disseminated dynamic geometry software \textit{GeoGebra} \cite{BHJKPRW} and \cite{Tutorial}.

Here we apply these criteria to our ``true on parts, false on parts'' context, as follows:

\begin{theorem}\label{testRV}
Let  $\{H  \Rightarrow  T\}$ be a geometric statement and fix a maximum-size set $Y=\{x_{1}, \dots, x_{d}\}$ of independent variables for the hypotheses ideal $H$ (i.e.~$d=\dim (H)$).
\begin{itemize}
\item[a)] The statement $\{H  \Rightarrow  T\}$ is  generally true if and only if 
$$\left<h_1,\dots , h_r, f\cdot t-1\right>K[X,t]\cap K[Y]\neq\left<0\right>.$$%\\[-22pt]
\item[b)] The statement $\{H  \Rightarrow  T\}$ is  generally false if and only if 
$$\left<h_1,\dots , h_r,f\right>K[X]\cap K[Y]\neq \left<0\right>.$$
\end{itemize}
\end{theorem}

\begin{proof}
The proof of a) is quite straightforward. In fact, the ideal 
$$\left<h_1,\dots , h_r, f\cdot t-1\right>K[X,t]\cap K[X]$$
\noindent is usually named, in commutative algebra, as the \emph{saturation} of $H$ by $f$ and it is well known (c.f.~\cite {DR}, Appendix, in particular, Remark 5) that it represents the intersection of the primary components of $H$ such that the associated primes do not contain $f$, i.e.~such that the thesis does not hold over the corresponding irreducible component.  It follows that if the intersection of this saturation with $K[Y]$ contains a non-zero polynomial $g\in K[Y]$, then all such ``failure" components must be degenerate, as they all contain $g$, a polynomial in the independent variables. 

And, conversely, if all the prime ideals associated to the primary ideals in the saturation are degenerate, there is a non-zero polynomial $g\in K[Y]$ in each of them and, then, a power of $g$ must be in the corresponding primary component. Thus, the product of all of these $g$'s is in the intersection of the primary ideals, so in the saturation, yielding 
$$\left<h_1,\dots , h_r, f\cdot t-1\right>K[X,t]\cap K[Y]\neq\left<0\right>.$$
\noindent  Notice this statement is true even if the independent variables set $Y$ is not of maximum-size. 

Now, we are going to make a detailed demonstration of b), since in \cite{RV99} it is just sketched as a footnote. 
For the ``if'' part of b), suppose there is a non-zero $g\in \left<h_1,\dots , h_r,f\right>K[X]\cap K[Y]$. Then $g=g_1 h_1+\cdots +g_r h_r+ g_{r+1} f$ for some $g_i\in K[X]$. Let $W$ be a non-degenerate component of $V(H)$. Then, $g$ cannot vanish on $W$, because $I(W)\cap K[Y]=\left<0\right>$. As $h_i$ vanishes on $W$ for all $i=1,\dots, r$, we have that $f$ must not vanish on all $W$. Remark that for this part of the proof we have not used that $Y$ is maximum-size.

For the proof of the ``only if'' of b) we have to assume  that $Y=\{x_{1}, \dots, x_{d}\}$ ($0< d\leq n$) is a maximum-size set of independent variables for the hypotheses ideal $H$. Suppose that $f$ vanishes on none of the non-degenerate components of $V(H)$.

Take any non-degenerate  component $W$ of $V(H)$  and let $\mathfrak{p}\subset K[X]$ be its associated prime ideal. Then, $\mathfrak{p}\cap K[Y]=\left<0\right>$. 
As $f$ does not vanish identically on $W$,  $\mathfrak{p}+\left<f\right>\supsetneq\mathfrak{p}$ and $\mathfrak{p}+\left<f\right>$  has  dimension less than $d$, by our maximum-size of $Y$ hypothesis. Then $\{x_{1},  \dots, x_{d}\}$ cannot be independent for $\mathfrak{p}+\left<f\right>$. Therefore, there is a non-zero  $g_W\in ( \mathfrak{p}+\left<f\right>) \cap K[Y]$ for each non-degenerate component $W$ of $V(H)$, vanishing over the intersection of this component and the thesis.
For each degenerate component $U$ of $V(H)$ take now a polynomial $g_U\in I(U)\cap K[Y]$.

Let $g$ be the product of all $g_W$ and all $g_U$. Then this product vanishes over all the points of $V(H)$ where the thesis holds, because it vanishes both over all degenerate components and over the zeroes of $f$ in the non-degenerate components. Thus   $g\in \sqrt{\left<h_1,\dots , h_r,f\right>}K[X]\cap K[Y]$ and therefore $\left<h_1,\dots , h_r,f\right>K[X]\in K[Y]\neq \left<0\right>$.
\end{proof}

Obviously, Theorem \ref{testRV} provides a straightforward test for detecting if an statement is true and false on components, by simply checking if the result of performing  two eliminations is zero in both cases or not. Moreover, several algorithms for elimination over algebraically closed fields, interpreted through a Gr\"obner basis computation, are already implemented in different computer algebra systems---including Giac,  the one currently embedded in the dynamic mathematics program \textit{GeoGebra}---and work satisfactorily for our purposes. 

This criterion is also useful to understand an apparently contradictory fact, since we have previously emphasized  the base field dependence of the primary decomposition of an ideal:

\begin{corollary}\label{extensionfield}
Suppose that we consider some intermediate base field extension $K\subseteq K'\subseteq L$,  where $L$ is algebraically closed.  Now, although $\{H  \Rightarrow  T\}$ is defined over $K$, consider this statement as well as defined over $K'$. Then we claim that  being \emph{true on parts, false on parts} over $K$ is equivalent to being  \emph{true on parts, false on parts} over $K'$, that is, the notion of \emph{true on parts, false on parts} does not depend on the base field extension. 
\end{corollary}
\begin{proof}In fact, deciding if some elimination ideal is zero or not can be achieved by computing a Gr\"obner basis under a certain ordering and checking if there are some elements in the basis that just involve the non-eliminated variables.  But Gr\"obner basis computations are performed over the base field, and, thus, are independent of the field extension (i.e.~it will yield the same result over $K$ or over $K'$). 
\end{proof}

 What this result means is that, if both hypotheses and the thesis are defined over a common base field $K$, then the existence of some component where the thesis vanishes or  where it does not vanish is independent of the choice of any field extension of $K$ as a new base field; we can have more components over $K$ where $T$ vanishes, and a few less over $K'$, and the same can happen for components where $T$ does not vanish; but their existence over $K$ is equivalent to their existence over $K'$. 

It could seem that, since we are concluding that the field extension is not relevant in this context,  the above remark directly implies that our framework is equivalent to that of \cite{ZWS}, in which it was supposed that $K=L$, algebraically closed.  But things are quite subtle here. In fact,  Theorem 3.1 in \cite{ZWS} gives a necessary and sufficient condition for a geometric statement to be generally true on components over $K=L$, but without assuming maximum-size for the set of independent variables $Y=\{x_1,\dots, x_d\}$,  as we did in our Theorem \ref{testRV}. In what follows we will confirm that this fact is important, see Example \ref{ex1}, but that, otherwise, the result of \cite{ZWS} also holds in our framework. 

 In order to prove this,  let all the notations be the same as above. In particular, let $X=\{x_1,\dots, x_n\}$ be the set of variables representing the point coordinates involved in the algebraic description of the hypotheses ideal, suppose that $Y=\{x_1,\dots, x_d\}$ ($d\leq n$) is a subset of independent variables for $H$, but not necessarily maximum-size,  and denote by $Z=\{x_{d+1},\dots, x_n\}$ the rest of the variables.  Then, we have: 

\begin{theorem}[Theorem 3.1 \cite{ZWS}]\label{testZWS}
Let $L=K$ be an algebraically closed field and let $H'=\left<h_1,\dots,h_r\right> K(Y)[Z]$ be the extension of $H$ to $K(Y)[Z]$. Then the geometric statement $\{H\Rightarrow T\}$ is generally true on components (i.e. \emph{true on parts, false on parts}) if and only if f is a non-zero zero divisor in $K(Y)[Z]/\sqrt{H'}$.
\end{theorem}

Next we give a direct proof of the equivalence between Theorems \ref{testRV} and \ref{testZWS} for being generally true on components in  a purely algebraic fashion, but taking a field $K$ and an algebraically closed extension $L$ and assuming  now that $Y$ is a maximum-size set of independent variables.

\begin{theorem}\label{equivalent}
Let  $\{H  \Rightarrow  T\}$ be a geometric statement as above and fix a maximum-size set $Y=\{x_{1}, \dots, x_{d}\}$ of independent variables for the hypotheses ideal $H$ (i.e.~$d=\dim (H)$) and let $Z=\{x_{d+1},\dots, x_n\}$ be the rest of the variables. Then, 
$$\left<h_1,\dots , h_r, f\cdot t-1\right>K[X,t]\cap K[Y]=\left<0\right> \text{ and } \left<h_1,\dots , h_r,f\right>K[X]\cap K[Y]=\left<0\right>$$ 
if and only if 
$$f \text{ is a non-zero zero divisor in } K(Y)[Z]/\sqrt{H'} \text{ where } H'=\left<h_1,\dots,h_r\right> K(Y)[Z].$$
\end{theorem}

We need some lemmas for the proof of this theorem. For the lemmas below we assume the previous notations and the hypotheses of Theorem \ref{equivalent}.

\begin{lemma}\label{lemma1}
Let $H'$ be  the extension of the hypotheses ideal $H$ to $K(Y)[Z]$.
\begin{itemize}
\item[a)]The ring  $K(Y)[Z]/\sqrt{H'}$ is not zero.
\item[b)]$\dim_{K(Y)}(K(Y)[Z]/\sqrt{H'})=0$. 
\end{itemize}
\end{lemma}

\begin{proof}
\begin{itemize}
\item[a)] The condition that $Y$ is independent for $H$ in $K[Y,Z]$ is equivalent to say that $H'$ is not all $K(Y)[Z]$ (i.e.~$1 \notin H'$ and, equivalently, $1 \notin \sqrt{H'}$)). So that the ring  $K(Y)[Z]/\sqrt{H'}$ is not zero.
\item[b)] Notice that the ideal  $H$ has dimension $d$ in $K[Y,Z]$ and $Y$ is a set of $d$ independent variables for $H$ and for $\sqrt{H}$.

For each variable $z \in Z$ we have that the variables in the set $\{Y,z\}$ are not independent for $H$. Then, there is a non-zero polynomial $g(Y,z)\in H\cap K[Y,z]$. 

Since $ H \subset HK(Y)[Z]=H'\subset \sqrt{H'}$,  for all $z\in Z$ there is a non-zero polynomial $g(Y,z)\in \sqrt{H'}K(Y)[Z]\cap K(Y)[z]$ and then $\dim_{K(Y)}(K(Y)[Z]/\sqrt{H})=0$. 
\end{itemize}

\end{proof}

As a consequence of previous lemma we have the following corollary.

\begin{corollary}\label{corollary}
Let $f$ be a polynomial in $K[Y,Z]$. If  $f\notin\sqrt{H'}$, then there is a polynomial $p(t)\in K(Y)[t]$ such that $p(f)\in \sqrt{H'}$.
\end{corollary}

\begin{proof}
Consider the primary decomposition of $\sqrt{H'}$ in the ring $K(Y)[Z]$. 
For each associated prime  $\mathfrak{p}$ of $\sqrt{H'}$, the field of fractions of  $K(Y)[Z]/\mathfrak{p}$  will be an algebraic extension of $K(Y)$, because $K(Y)[Z]/\mathfrak{p}$ has dimension 0 over $K(Y)$ (by previous lemma). Thus, for each prime ideal $\mathfrak{p}$ associated to $\sqrt{H'}$, there is a polynomial $p_{\mathfrak{p}}(t)\in K(Y)[t]$ such that $p_{\mathfrak{p}}(f)\in \mathfrak{p}$. 

Take then $p(t)=\prod_{\mathfrak{p}} p_{\mathfrak{p}}(t)\in K(Y)[t]$. Then, $p(f)\in \cap\mathfrak{p}=\sqrt{H'}$.
\end{proof}

\begin{lemma}\label{lemma2}
$f\in \sqrt{H'}$ if and only if $\left<h_1,\dots , h_r, f\cdot t-1\right>K[Y,Z,t]\cap K[Y]\neq\left<0\right>$
\end{lemma}

\begin{proof}
$f\in \sqrt{H'}$ means that there is a positive integer $m$ such that $f^m=\sum_{i=1}^r k_i' h_i$ with $k_i'\in K(Y)[Z]$. Equivalently, by clearing denominators, there is a positive integer $m$ such that $g f^m=\sum_{i=1}^r k_i h_i$ where $g\in K[Y]$ and $k_i\in K[Y,Z]$ (i.e., $g f^m\in H$).
That is, $g\in (H:f^\infty)\cap K[Y]$  where $(H:f^\infty)$ is the saturation of the ideal $H$ by $f$.

But $(H:f^\infty)=\left<h_1,\dots , h_r,f\cdot t-1\right>K[Y,Z,t]\cap K[Y,Z]$ (c.f.~\cite{DR}, Appendix, Proposition 6), then $\left<h_1,\dots , h_r, f\cdot t-1\right>K[Y,Z,t]\cap K[Y]\neq\left<0\right>$.
\end{proof}

Let us prove Theorem \ref{equivalent}.

\begin{proof}

To prove the ``if'' part, let assume that $f$ is a non-zero zero divisor in $K(Y)[Z]/\sqrt{H'}$. As $f$ is non-zero in $K(Y)[Z]/\sqrt{H'}$, $f\notin\sqrt{H'}$. Then, by Lemma \ref{lemma2}, we have that  
$$\left<h_1,\dots , h_r, f\cdot t-1\right>K[Y,Z,t]\cap K[Y]=\left<0\right>.$$

Suppose also that  $\left<h_1,\dots , h_r, f\right>K[Y,Z]\cap K[Y]\neq \left<0\right>$, then $(\sqrt{H}+\left<f\right>)K[Y,Z]\cap K[Y]\neq \left<0\right>$. Therefore, there is a polynomial $g\in K[Y]$, such that $g=k+qf$ with $k\in \sqrt{H}$ and $q\in K[Y,Z]$. Dividing by $g$ we obtain 
$$1=\frac{k}{g}+\frac{q}{g}f$$
\noindent and notice that $\frac{k}{g}\in \sqrt{H'}K(Y)[Z]$ and $\frac{q}{g}\in K(Y)[Z]$. So $f$ is a unit in $K(Y)[Z]/\sqrt{H'}$. Then, $f$ cannot be a zero divisor in $K(Y)[Z]/\sqrt{H'}$.

Now let us prove the ``only if'' part. Assume that 
$$\left<h_1,\dots , h_r, f\cdot t-1\right>K[Y,Z,t]\cap K[Y]=\left<0\right> \text{ and } \left<h_1,\dots , h_r, f\right>K[Y,Z]\cap K[Y]=\left<0\right>.$$ 
Then, by Lemma \ref{lemma2}, $f\notin \sqrt{H'}$ and, by Corollary \ref{corollary}, there is a polynomial $p(t)\in K(Y)[t]$ such that $p(f)\in \sqrt{H'}$. 

Take a polynomial $p'\in K(Y)[t]$ of minimum degree in $t$ with this property. Then we have two possibilities:
\begin{enumerate}[(i)]
\item $p'$ has an independent term in $K(Y)$: as $p'(f) \in \sqrt{H'}$, take a convenient power of $p'(f)$ and clearing denominators, we will obtain an independent term in $K[Y]$ that will be a combination of  $h_1,\dots,h_r$ and $f$. Thus,   $\left<h_1,\dots , h_r, f\right>K[Y,Z]\cap K[Y]\neq\left<0\right>$, in contradiction with one of our hypotheses.
\item $p'$ does not have an independent term: we can take $f$ as a common factor in $p'(f) \in \sqrt{H'}$, yielding  $p'(f)=f\cdot q(f) \in\sqrt{H'}$.  Moreover, $q(f)$ cannot be in $\sqrt{H'}$, because it has a lower degree than $p'(t)$.  Thus we will have  $f\cdot q(f)=0$ in $K(Y)[X]/\sqrt{H'}$ and $q(f)$ non-zero in $K(Y)[X]/\sqrt{H'}$. Therefore, $f$ is a zero divisor in $K(Y)[X]/\sqrt{H'}$.
\end{enumerate}
\end{proof}

We have proved that the tests in \cite{RV99}  and  \cite{ZWS} for being true on components  are equivalent, even in our generalized context, although requiring a maximum-size set of independent variables. Consequently, we have now a new  algorithm (based on the direct application of  Theorem \ref{testRV}) to check if a geometric statement is true on components, by merely  using elimination in polynomial ideals,  in contrast with the one presented in \cite{ZWS} that requires computing a Gr\"obner basis over a field of fractions, and checking if $f$ is a zero divisor of the radical of some ideal in a extended ring. 

To illustrate this new test, we have chosen Example 3.8 in \cite{ZWS}, where truth on components is checked by the zero divisor test. We apply here our test  using the computer algebra system Maple 2017.0, although it can be computed in whatever system doing elimination in polynomial ideals. The computations  have been performed in a few seconds.

\begin{example}\label{ex3.8}
 Let $ABC$ be a triangle with $A(0,0)$, $B(1,0)$ and $C(u_1,u_2)$, and let $A_1BC$, $AB_1C$ and $ABC_1$ be three equilateral triangles  erected on the three sides of $ABC$. Then, check if the  segments $B_1C_1$ and $A_1C$ have the same length, that is, $|B_1C_1|=|A_1C|$.

Taking coordinates $A_1=(x_1,x_2)$, $B_1=(x_3,x_4)$ and $C_1=(x_5,x_6)$, the hypotheses ideal is given by  the following polynomials describing the equalities between the sides of the three equilateral triangles:

$$\begin{array}{ll}
|AC_1|=|AB|: \;\; & h_1= x_5^2+x_6^2-1 \\
|BC_1|=|AB|: \;\; & h_2= (x_5 - 1)^2 + x_6^2- 1 \\
|CA_1|=|BC|: \;\; & h_3= (x_1 - u_1)^2 + (x_2 - u_2)^2 - (u_1 - 1)^2 - u_2^2 \\
|BA_1|=|BC|: \;\; & h_4= (x_1 - 1)^2 + x_2^2- (u_1 - 1)^2 - u_2^2  \\
|AB_1|=|AC|: \;\; & h_5= (x_3^2+ x_4^2) - (u_1^2+u_2^2) \\
|CB_1|=|AC|: \;\; & h_6=  (x_3 - u_1)^2 + (x_4 - u_2)^2 - (u_1^2+ u_2^2).
\end{array}$$

And the thesis $|B_1C_1|=|A_1C|$ is given by the polynomial
$$f:=(x_5 - x_3)^2 + (x_6 - x_4)^2 - (x_1 - u_1)^2 - (x_2 - u_2)^2.$$

The set of variables $X=\{u_1,u_2,x_1,x_2,x_3,x_4,x_5,x_6\}$ has 8 elements.  By the geometric construction, the variables  $\{u_1,u_2\}$ can be considered as the  free variables in the hypotheses ideal $H$. In fact, one can check  that they are even a maximum-size set of independent variables for $H$. In Maple, after downloading the package \emph{PolynomialIdeals} and defining the ideal
$$\begin{array}{l}
{H:=\left<x_5^2+x_6^2-1,(x_5 - 1)^2 + x_6^2- 1,\right.}\\
{(x_1 - u_1)^2 + (x_2 - u_2)^2 - (u_1 - 1)^2 - u_2^2,(x_1 - 1)^2 + x_2^2- (u_1 - 1)^2 - u_2^2},\\
{\left.(x_3^2+ x_4^2) - (u_1^2+u_2^2), (x_3 - u_1)^2 + (x_4 - u_2)^2 - (u_1^2+ u_2^2)\right>},
\end{array}$$
\noindent we compute its Hilbert dimension (by using the command \texttt{HilbertDimension($H$)}), yielding that it is 2 and (with the help of command \texttt{MaximalIndependentSet($H$)}) that $\{u_1,u_2\}$ is, as expected, a maximum-size set of independent variables.

Then, we check if the statement is true on parts, false on parts,  by eliminating all variables except $\{u_1,u_2\}$ in the ideals $\left<h_1,\dots,h_6, f\cdot t-1\right>$ and $\left<h_1,\dots,h_6, f\right>$.
$$\begin{array}{l}
{\texttt{EliminationIdeal}(\left<x_5^2+x_6^2-1,(x_5 - 1)^2 + x_6^2- 1,\right.}\\[4pt]
{(x_1 - u_1)^2 + (x_2 - u_2)^2 - (u_1 - 1)^2 - u_2^2,(x_1 - 1)^2 + x_2^2- (u_1 - 1)^2 - u_2^2,}\\[4pt]
{(x_3^2+ x_4^2) - (u_1^2+u_2^2), (x_3 - u_1)^2 + (x_4 - u_2)^2 - (u_1^2+ u_2^2),}\\[4pt] 
{\left.((x_5 - x_3)^2 + (x_6 - x_4)^2 - (x_1 - u_1)^2 - (x_2 - u_2)^2)\cdot t-1\right>,\{u_1, u_2\});}\\[4pt]
{ \hspace{4cm}                             \left<\, 0\right>}\\[6pt]
{\texttt{EliminationIdeal}(\left<x_5^2+x_6^2-1,(x_5 - 1)^2 + x_6^2- 1,\right.}\\[4pt]
{(x_1 - u_1)^2 + (x_2 - u_2)^2 - (u_1 - 1)^2 - u_2^2,(x_1 - 1)^2 + x_2^2- (u_1 - 1)^2 - u_2^2,}\\[4pt]
{(x_3^2+ x_4^2) - (u_1^2+u_2^2), (x_3 - u_1)^2 + (x_4 - u_2)^2 - (u_1^2+ u_2^2),}\\[4pt] 
{\left.(x_5 - x_3)^2 + (x_6 - x_4)^2 - (x_1 - u_1)^2 - (x_2 - u_2)^2)\right>,\{u_1, u_2\});}\\[4pt]
{  \hspace{4cm}                             \left<\, 0\right>}
\end{array}$$
We obtain that both eliminations give the $\left<0\right>$ ideal, and we conclude, by the elimination test, that this statement is true and false on certain components, i.e. \emph{true on parts, false on parts}.

\end{example}

We would like to point out that in Example 3.8 of \cite{ZWS}, reproduced above, the authors also consider  $\{u_1,u_2\}$ as the set of choice for independent variables, which is, as previously remarked,  a maximum-size set of independent variables. The following example shows that the elimination test we have presented and the zero divisor test of \cite{ZWS}  do not agree if the number of independent variables is not maximum-size. Moreover, it explains also some pretended error detected by \cite{ZWS}, Remark 3.6, concerning a discrepancy with \cite{MR}.

\begin{example}\label{ex1}
Consider $H=\left<xy,x^2\right>$ and $f=y$ in $K[x,y,z]$.  The ideal $\left<xy,x^2\right>K[x,y,z]$ has dimension $2$ and $\{y,z\}$ is a maximum-size set of independent variables. Take, instead,  $Y=\{z\}$,  which is a set of independent variables, but not maximum-size. Then, 
$$
\left<xy,x^2,y\cdot t-1\right>K[x,y,z,t]\cap K[z]=\left<0\right>\; \mathrm{ and }\;
\left<xy,x^2,y\right>K[x,y,z]\cap K[z]=\left<0\right>.$$
\noindent So if we apply the test \cite{RV99} in this specific situation in which one does not consider all possible independent variables, the statement seems to be \emph{true on parts, false on parts} (i.e. true on components).

But on the other hand, $y$ is not  a zero divisor in $K(z)[x,y]/\sqrt{\left<xy,x^2\right>}$. In fact, $\sqrt{\left<xy,x^2\right>}=(x)$, so  $K(z)[x,y]/\sqrt{\left<xy,x^2\right>}= K(z)[x,y]/(x)= K(z)[y]$ which is a domain of integrity and it does not contain  zero divisors. 
Therefore, the statement is not true on components, according to the test \cite{ZWS}. Actually the statement is generally false, because $y$ is not zero over the plane $\{x=0\}$.
\end{example}

\begin{remark}
Just by following the proof of Theorem \ref{testRV}, it is easy to conclude that, without any maximum-size restriction for the independent variables, it holds that to be true on components (and, thus, verifying the criteria of Theorem \ref{testZWS})  implies the two conditions stated in Theorem \ref{testRV}, but not conversely, as the above example shows.  So, without the maximum-size restriction, our test is just a necessary condition for a true on components situation, not a sufficient one (but it becomes sufficient for sets of independent variables of  maximum-size).
\end{remark}

\section{Examples in dynamic geometry}
\label{sec3}
In this section we refer to the dynamic geometry system \textit{GeoGebra},  which deals
with the concept of \emph{true on parts, false on parts} since version 5.0.443.0 (10 March 2018). Thus, statements which were formerly classified as  \emph{not generally true} and, therefore,  simply considered as \emph{false}
by \textit{GeoGebra}, are  now subject to a finer classification (for instance, some of them can now be classified as \emph{true on parts, false on parts})   yielding a more accurate information to the user.

The implementation in  \textit{GeoGebra} of this new feature follows Theorem \ref{testRV} by computing both eliminations and deciding
if the ``intuitively''\footnote{But automatically  chosen by the program. See \url{https://dev.geogebra.org/trac/browser/trunk/geogebra/common/src/main/java/org/geogebra/common/kernel/prover/HilbertDimension.java} for details.} full set of independent variables is actually maximum-size, by determining the Hilbert
dimension of the hypotheses
ideal. We outline the steps of this algorithm  to decide the truth/falsity of a statement $\{H  \Rightarrow  T\}$:
\begin{enumerate}
\item First select,  through the coordinates of the free points of the configuration and following the construction steps performed by the user when drawing the figure that illustrates the given statement, a  set  $Y$  of  independent variables. Check if they  are actually algebraically independent, $H\cap K[Y]=\left<0\right>$.

\item Verify if the Hilbert's dimension of $H$ agrees with the cardinality  of $Y$. If this is not the case, the user is advised to \textit{check for  degenerations in the construction} (END). Otherwise, continue.

\item  Compute $\left<h_1,\dots , h_r, f\cdot t-1\right>K[X,t]\cap K[Y]$. If it is  $\neq\left<0\right>$,  the statement is \textit{generally true} (END). Otherwise, continue.

\item  Compute $\left<h_1,\dots , h_r, f\right>K[X]\cap K[Y]$. If it is  $\neq\left<0\right>$,  the statement is \textit{generally false} (END). 

\item  Otherwise, the statement is  \emph{true on parts, false on parts}.
\end{enumerate}

\begin{example}
Assume two points $A$ and $B$ are given. A rhombus $ABDC$ is to be constructed, as displayed in Fig.~\ref{rhombus},  by
allowing point $C$ to be freely chosen under the restriction that the segments $AB(=f)$ and $AC(=g)$ are equal.
That is, $C$ is a circumpoint of the circle with center $A$ and radius $f$, and $D$ is an intersection
of a circle $c$ with center $C$ and radius $f$, and a line $h$ through $C$ and parallel to $f$.

Now we consider the well-known statement that the diagonals of a rhombus are perpendicular, but this
fact (namely, $k\perp l$) can only be
observed on the component of the hypotheses variety that contains $D$. For the other component, where $D'$  lies (which is
the other intersection point of circle $c$ and line $h$), the fact $k\parallel l$ can be detected.

\begin{figure}[h]
\includegraphics[scale=0.9]{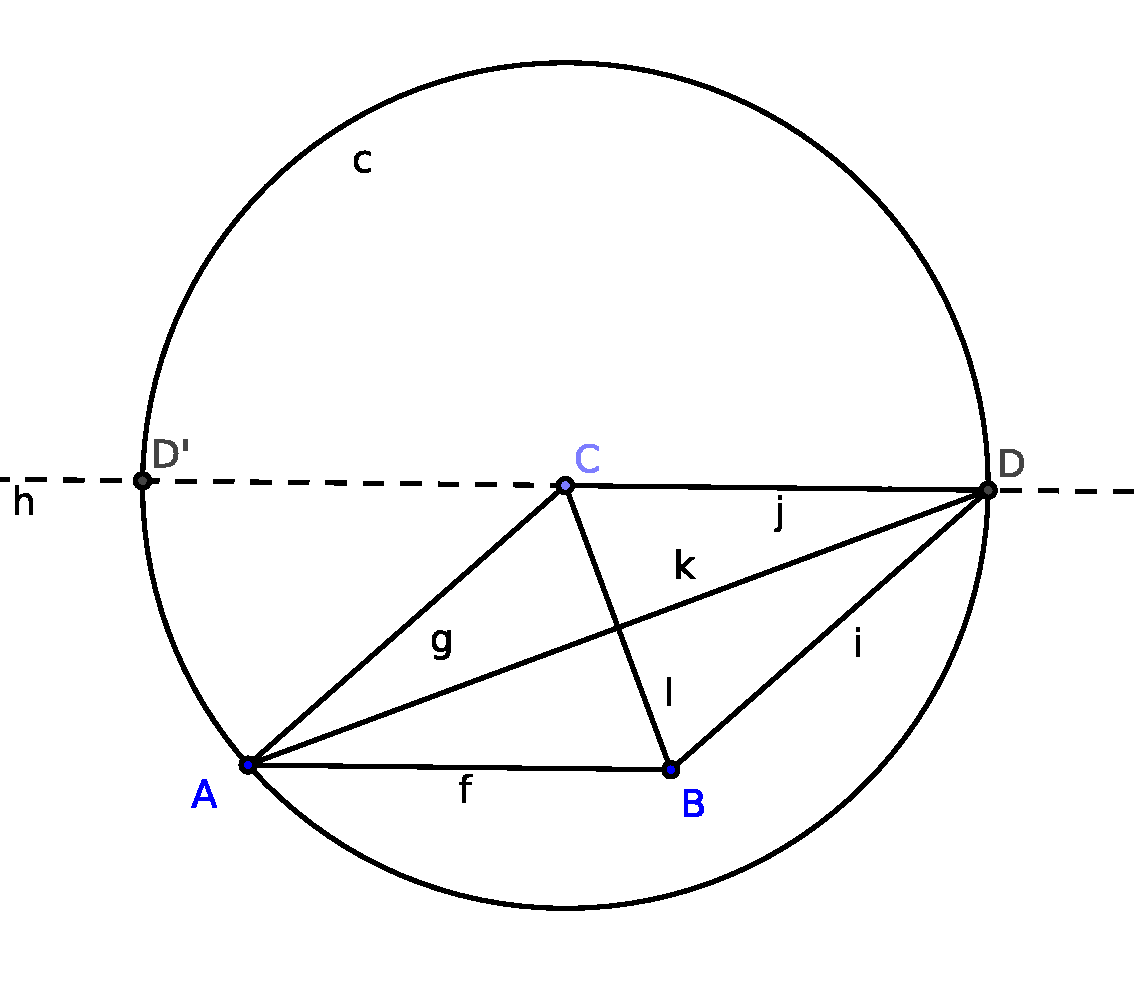}\caption{Constructing a rhombus (an example provided by Andreas Lindner)}
\label{rhombus}
\end{figure}

To support the first glance, \textit{GeoGebra} provides a numerical way to verify whether the perpendicular property is true.
When comparing objects $k$ and $l$ by using the \textit{Relation tool} in \textit{GeoGebra}, after a numerical overview
(Fig.~\ref{RelNum}) a symbolic proof (Fig.~\ref{RelSymb}) can be achieved yielding that the statement
is ``true on parts, false on parts". (The algebraic details of
the proof are not visible for the user,  to avoid confusion.) This kind of result can also be automatically obtained
by using \textit{GeoGebra}'s low-level \texttt{ProveDetails[$k\perp l$]} command---and also for the parallelism, the \texttt{ProveDetails[$k\parallel l$]} command.
By getting \texttt{\{true, "c"\}} as output we are warned that these results are true on some \textit{components} only.

\begin{figure}[h]
\includegraphics[scale=0.4]{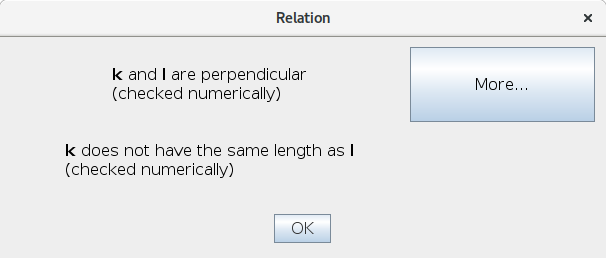}\caption{A numerical approach for detecting truth in \textit{GeoGebra}}
\label{RelNum}       
\end{figure}

\begin{figure}[h]
\includegraphics[scale=0.4]{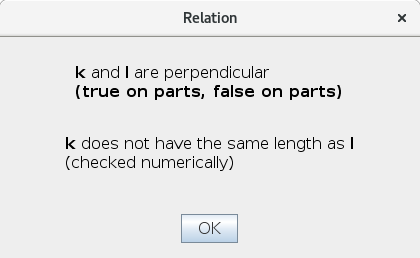}\caption{A symbolic approach for detecting truth in \textit{GeoGebra}}
\label{RelSymb}       
\end{figure}

Let us summarily consider the algebraic counterpart of this geometric construction, that is, how \textit{GeoGebra} automatically
sets up the input polynomials for running the described algorithm. Free points are defined with coordinates
$A(v_1,v_2)$, $B(v_3,v_4)$. Since point $C$ is defined as a point on a circle with center $A$ and radius $f$ (that
is the segment $AB$), \textit{GeoGebra} introduces a hidden technical point $X_1(v_5,v_6)$ to describe the vector $\overrightarrow{AB}$ by using
equations
\begin{equation}
h_1=v_{5}-v_{3}=0,\label{ee1}
\end{equation}
\begin{equation}
h_2=v_{6}-v_{4}=0,\label{ee2}
\end{equation}
and then, the constrained point $C(v_7,v_8)$ is given by the equation
\begin{equation}
h_3=-v_{8}^{2}-v_{7}^{2}+v_{6}^{2}+v_{5}^{2}+2v_{8}v_{2}-2v_{6}v_{2}+2v_{7}v_{1}-2v_{5}v_{1}=0.\label{ee3}
\end{equation}
Now line $h$ can be described as going through $C$ and parallel to $f$, by implicitly creating hidden technical point $X_2(v_9,v_{10})$ such that the parallel line joins $C$ and $X_2$:
\begin{equation}
h_4=v_{9}-v_{7}-v_{3}+v_{1}=0,\label{ee4}
\end{equation}
\begin{equation}
h_5=v_{10}-v_{8}-v_{4}+v_{2}=0,\label{ee5}
\end{equation}

Another technical point $X_3(v_{11},v_{12})$ is introduced as a circumpoint of circle $c$ with center $C$ and radius $f$
with the help of equations
\begin{equation}
h_6=v_{11}-v_{7}-v_{3}+v_{1}=0,\label{ee6}
\end{equation}
\begin{equation}
h_7=v_{12}-v_{8}-v_{4}+v_{2}=0.\label{ee7}
\end{equation}
As the final step to describe the hypotheses the intersection point $D(v_{13},v_{14})$ of line $h$ and circle $c$ is defined by
\begin{equation}
h_8=v_{13}v_{10}-v_{14}v_{9}-v_{13}v_{8}+v_{9}v_{8}+v_{14}v_{7}-v_{10}v_{7}=0,\label{ee8}
\end{equation}
\begin{equation}
h_9=-v_{14}^{2}-v_{13}^{2}+v_{12}^{2}+v_{11}^{2}+2v_{14}v_{8}-2v_{12}v_{8}+2v_{13}v_{7}-2v_{11}v_{7}=0.\label{ee9}
\end{equation}

The thesis equation is
\begin{equation}
f=v_{14}v_8+v_{13}v_7-v_{14}v_4-v_{13}v_3-v_8v_2+v_4v_2-v_7v_1+v_3v_1=0.\label{et}
\end{equation}

Without loss of generality \textit{GeoGebra} assumes that $v_1=v_2=0$. After performing these substitutions and
using the notations from Theorem \ref{testRV}, we have $$X=\{v_3,v_4,v_5,v_6,v_7,v_8,v_9,v_{10},v_{11},v_{12},v_{13},v_{14}\},$$ 
$$Y=\{v_3,v_4,v_7\},\ K=\mathbb{Q},$$ and \textit{GeoGebra} computes 
$$\left<h_1,\dots , h_9, f\cdot t-1\right>K[X,t]\cap K[Y]$$%
and 
$$\left<h_1,\dots , h_9,f\right>K[X]\cap K[Y].$$
Since both are $\left<0\right>$ and the computed Hilbert dimension of $\left<h_1,h_2,\dots,h_9\right>$  is 3,
equal to $|Y|$, the statement is identified as ``true on parts, false on parts''. The result is computed by \textit{GeoGebra}
below 1 second on a modern PC.

This \textit{GeoGebra} example can also be tried out online at \url{https://www.geogebra.org/m/VeAxJHmS}. Clearly, the
number of used variables and equations is an overkill here, but since it is organized completely automatically by the
software translating
the geometry statement to an algebraic setup, it can be acceptable. Also other variables might be
substituted to some concrete integer numbers, e.g.~$v_3=0,v_4=1$, yielding further simplifications.
Such issues should be thoroughly addressed in future versions of \textit{GeoGebra}.
\end{example}

Further examples can be found in \textit{GeoGebra}'s automated benchmarking database at
\url{https://prover-test.geogebra.org/job/GeoGebra-provertest/566/artifact/test/scripts/benchmark/prover/html/all.html}
as of 24th March,  2018.
All database cells that contain the single character
\texttt{c} refer to an identified case of true on \textit{components}.

\section{Final reflections}
At first glance it could seem that we this paper deals with an algorithm for detecting  geometric statements that are true on some instances and false in some others\dots But, indeed,  surely most affirmations are like this: holding in some very particular cases, failing in most of them. What could be the interest of  automatically confirming the status of such irrelevant assertions?  

First of all, let us remark that we are \emph{not} dealing with all statements having this bivalent condition: we are focusing on those statements such that both the set of instances where the statement holds and the set where it fails are quite large and, in many respects,  geometrically meaningful. Indeed, it happens, in general, that the irreducible components of the hypotheses variety carry some special geometric significance.

Thus, as argued in Section 2 with references and examples,  detecting a \emph{true on parts, false on parts} statement provides some interesting and intriguing information on the geometry of the given problem, yielding sometimes, with the help of human intelligence, to the discovery of relevant geometric facts; and, in all cases, rising a warning sign for the  human user  on the need to start thinking there could be ``something" relevant behind!

For instance, at a very basic level, this \emph{true on parts, false on parts} situation could point out to the need to revise  some constructions  steps, that could indeed have got to be improved to avoid the bivalent behavior of a given statement in the different components associated to the construction. Thus, in Example 3.7  in \cite{ZWS}, if, instead of using circles (yielding to a a \emph{true on components} conclusion) , the user  deals just with rotations by 90 degrees, a clearly true statement is obtained.

On the other hand, some constructions cannot be improved by choosing a different approach. For instance,
 Example \ref{ex3.8} (also in \cite{ZWS} Example 3.8) deals with two potential constructions of equilateral triangles over each side of a given triangle, undistinguishable, at least, from the complex algebraic geometry approach. That is,
in dealing with such statements the notion of ``true on components'' is unavoidable.

In \textit{GeoGebra}'s Automated Reasoning Tools, where we have performed the examples in Section \ref{sec3},   the case \emph{true on parts, false on parts} is considered as a particular case of truth (true, but on some components only, the user is warned!).
We need to admit that,  from a rigorous point of view, this case could also be considered
as a particular kind of failure. Are the segments $k$ and $l$ in Fig.~\ref{rhombus} perpendicular if $D$ is
defined as an intersection of line $h$ and circle $c$? ``Well, no!''---that would
be the answer of a rigorous maths professor. It seems however more supporting and creative for the student (and the researcher), even if the statement is, strictly speaking, not always true,   
to get a partially yes-answer,  mentioning  that some additional
hypothesis---to be discovered by the user---may be required to achieve the complete truth of the investigated statement.

Finally,  we would like to emphasize that this approach, in our opinion, can also be particularly
fruitful in an educational context. Automated classification of statements
can be helpful in homework for self-evaluation of the
student; or in automated exam correction, helping the teacher. In fact, in the previous sections  we have already
shown  several elementary statements in Euclidean planar geometry that
could be identified as ``true on components''.  

In our opinion,  and as shown by the examples and references mentioned in Section 2, this subtle notion of \emph{true on parts, false on parts} is both interesting for the researcher, as a powerful tool for discovery,  and for the student, and can rise  even when considering statements of very basic  geometry.  Thus, it deserves to be functional in whatever automatic geometry reasoning program aiming to be considered fully useful for researchers and students. The algorithm and implementation we have presented in this paper could be considered as a first, but already fully performing,  step in this direction.

%\begin{acknowledgements}
%\end{acknowledgements}

\end{document}